\documentclass{article}

\usepackage{PRIMEarxiv}

\usepackage[utf8]{inputenc} 
\usepackage[T1]{fontenc}    
\usepackage{hyperref}       
\usepackage{url}            
\usepackage{booktabs}       
\usepackage{amsfonts}       
\usepackage{nicefrac}       
\usepackage{microtype}      
\usepackage{lipsum}
\usepackage{fancyhdr}       
\usepackage{graphicx}       
\graphicspath{{media/}}     
\usepackage{caption}
\usepackage{subcaption}
\usepackage{bbm}
\usepackage{amsmath,amssymb}
\usepackage{amsthm}
\usepackage{algorithm}
\usepackage{algpseudocode}
\usepackage[export]{adjustbox}
\newtheorem{theorem}{Theorem}
\newtheorem{corollary}{Corollary}[theorem]

\theoremstyle{plain}

\title{Intersectional Fairness: A Fractal Approach
}

\author{
  Giulio Filippi$^*$, Sara Zannone$^*$, Adriano Koshiyama \\
  Holistic AI \\
  London \\
  UK\\
}

\begin{document}

\maketitle
\def\thefootnote{*}\footnotetext{These authors contributed equally to this work. contact: sara.zannone@holisticai.com.}\def\thefootnote{\arabic{footnote}}

\begin{abstract}
The issue of fairness in AI has received an increasing amount of attention in recent years. The problem can be approached by looking at different protected attributes (e.g., ethnicity, gender, etc) independently, but fairness for individual protected attributes does not imply intersectional fairness. In this work, we frame the problem of intersectional fairness within a geometrical setting. We project our data onto a hypercube, and split the analysis of fairness by levels, where each level encodes the number of protected attributes we are intersecting over. We prove mathematically that, while fairness does not propagate "down" the levels, it does propagate "up" the levels. This means that ensuring fairness for all subgroups at the lowest intersectional level (e.g., black women, white women, black men and white men), will necessarily result in fairness for all the above levels, including each of the protected attributes (e.g., ethnicity and gender) taken independently. We also derive a formula describing the variance of the set of estimated success rates on each level, under the assumption of perfect fairness. Using this theoretical finding as a benchmark, we define a family of metrics which capture overall intersectional bias. Finally, we propose that fairness can be metaphorically thought of as a “fractal” problem. In fractals, patterns at the smallest scale repeat at a larger scale. We see from this example that tackling the problem at the lowest possible level, in a bottom-up manner, leads to the natural emergence of fair AI. We suggest that trustworthiness is necessarily an emergent, fractal and relational property of the AI system. 
\end{abstract}


\section{Introduction}


The issue of fairness in AI has received an increasing amount of attention in recent years. A number of AI systems involved in sensitive applications, like recruitment or credit scoring,  were found to be biased against minority groups \cite{amazon, apple}. For these reasons, the machine learning research community has focused on finding solutions to reduce the bias found in these models \cite{agarwal2018reductions, Agarwal2019, li2019repair, mehrabi2021survey}. Many of these solutions focused on comparing either the outputs of the model (Equality of Outcome \cite{feldman2015certifying}) or the error made by the algorithm (Equality of Opportunity \cite{hardt2016equality}) for different groups. These groups are classically identified by splitting the population using individual protected attributes, for instance according to gender, ethnicity or age.  \\ 

However, what happens when these attributes intersect? The concept of intersectional fairness, initially introduced by Black feminist scholars \cite{crenshaw2013demarginalizing, crenshaw2013mapping}, has only recently been introduced in the context of AI. In their seminal work, Buolamwini and Gebru \cite{buolamwini2018gender} performed a thorough analysis of three of the most popular facial recognition algorithms on the market (Microsoft, IBM, and Face++). They found that all algorithms were better at recognizing men and people with lighter skin, but the performance drastically decreased for darker-skinned females. The fact that the discrimination faced by Black women was “greater than the sum” of the discrimination experienced by Black men and white women is a well-known concept in the feminist literature \cite{crenshaw2013demarginalizing}, also known as fairness gerrymandering in the AI literature \cite{Kearns2018}. \\ 

These results have sparked an interest and motivated the need for studying the intersectional fairness of AI algorithms. Practically, this has meant extending the study of bias from individual protected attributes to their intersectional groups. Most of the works on intersectional fairness thus far has been concerned with measuring bias when intersecting all the protected attributes \cite{Foulds2020, foulds2020bayesian, jin2020mithracoverage, Morina2019}. Since every dataset may include an arbitrary number of protected attributes, which can be combined in any number of ways, it is not clear how one should pick the right level of granularity for the analysis \cite{Kong2022}.  \\  

In our paper, we tackle the problem from a slightly different perspective. Instead of fixing the level of analysis a priori, we develop a framework that allows us to examine all possible intersectional groups, at all levels of granularity, and their relationship to each other. We frame the problem of intersectional fairness using a geometrical structure, specifically a hypercube, where each protected attribute can be seen as a different dimension of the hypercube. We then project our data on the hypercube and use it to analyse how bias-related metrics like success rate or accuracy vary for different subgroups. The geometrical and mechanistic nature of our framework also allows us to compute the fairness of all possible subgroups in a computationally efficient way.     \\

This allows us to see how the fairness properties vary for different levels. From our framework it is easy to see that fairness indeed does not propagate downwards (i.e., gerrymandering). However, we prove that it does propagate upwards. This means that ensuring fairness at the most granular level, the intersection of all protected attributes in the dataset, implies fairness at any possible intersection. We will show that these results hold for both Equality of Outcome and Equality of Opportunity frameworks. This mirrors and extends the results of \cite{Foulds2020, Morina2019, Yang2020}. Finally, we find that, under perfect fairness, the variance of the set of estimated success rates on a given level decreases exponentially with increasing levels. We use this formula as a benchmark, and propose a family of metrics which capture overall intersectional bias.\\

We argue that the novelty of our work lies especially in the fact that our framework provides a tool to observe fairness for all the possible intersectional subgroups at once, and how their fairness properties interlink. 
Not only this could be great for visualisation, but it could be useful for modelling, possibly improving a mechanistic understanding of intersectional fairness. \\


\section{Problem Setting}
\begin{figure}
\begin{subfigure}{0.45\linewidth}
\centering
\includegraphics[width=4.5cm]{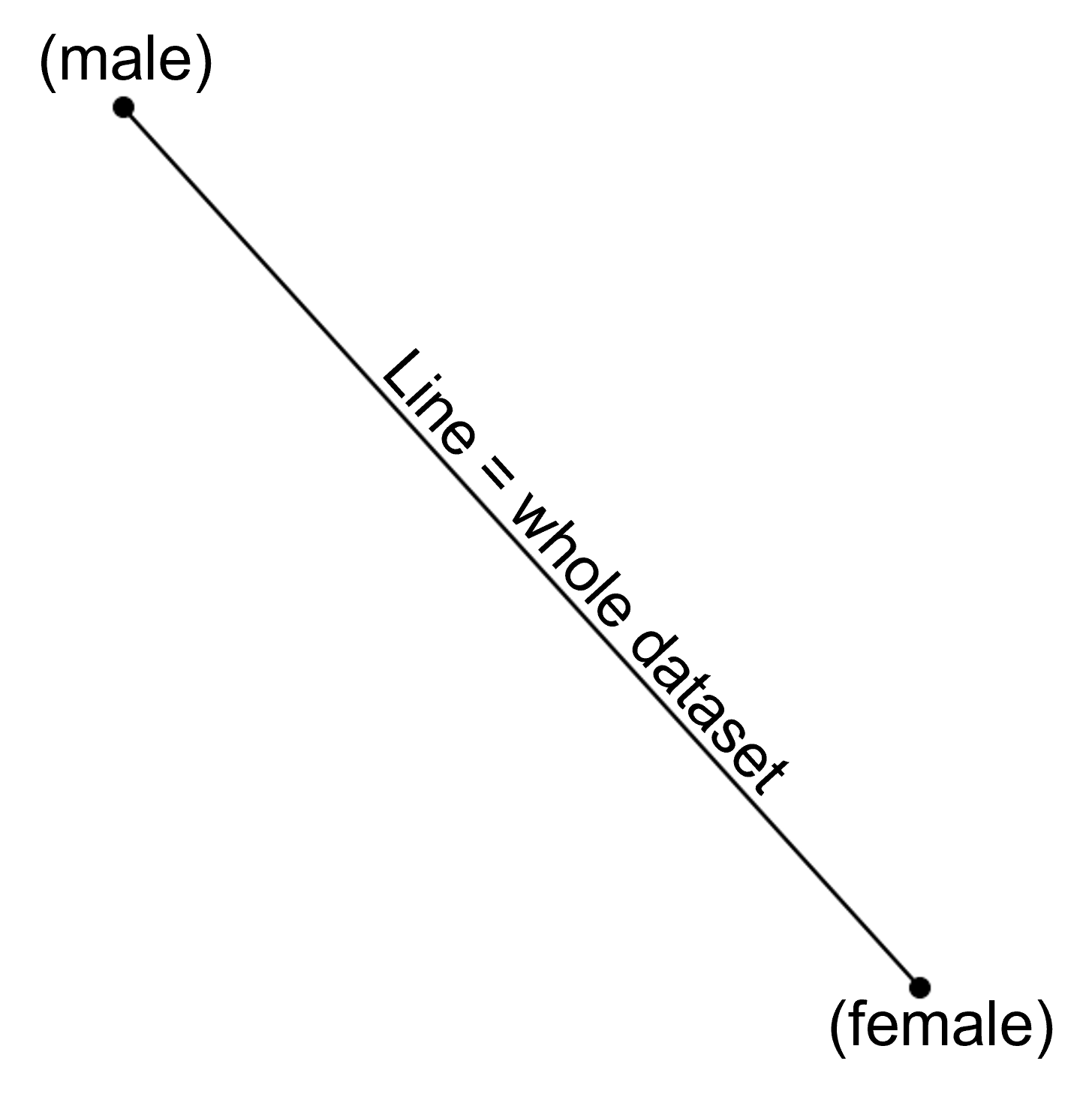}
\caption{}
\label{fig:4dcubea}
\end{subfigure}
\begin{subfigure}{0.45\linewidth}
\centering
\includegraphics[width=6cm]{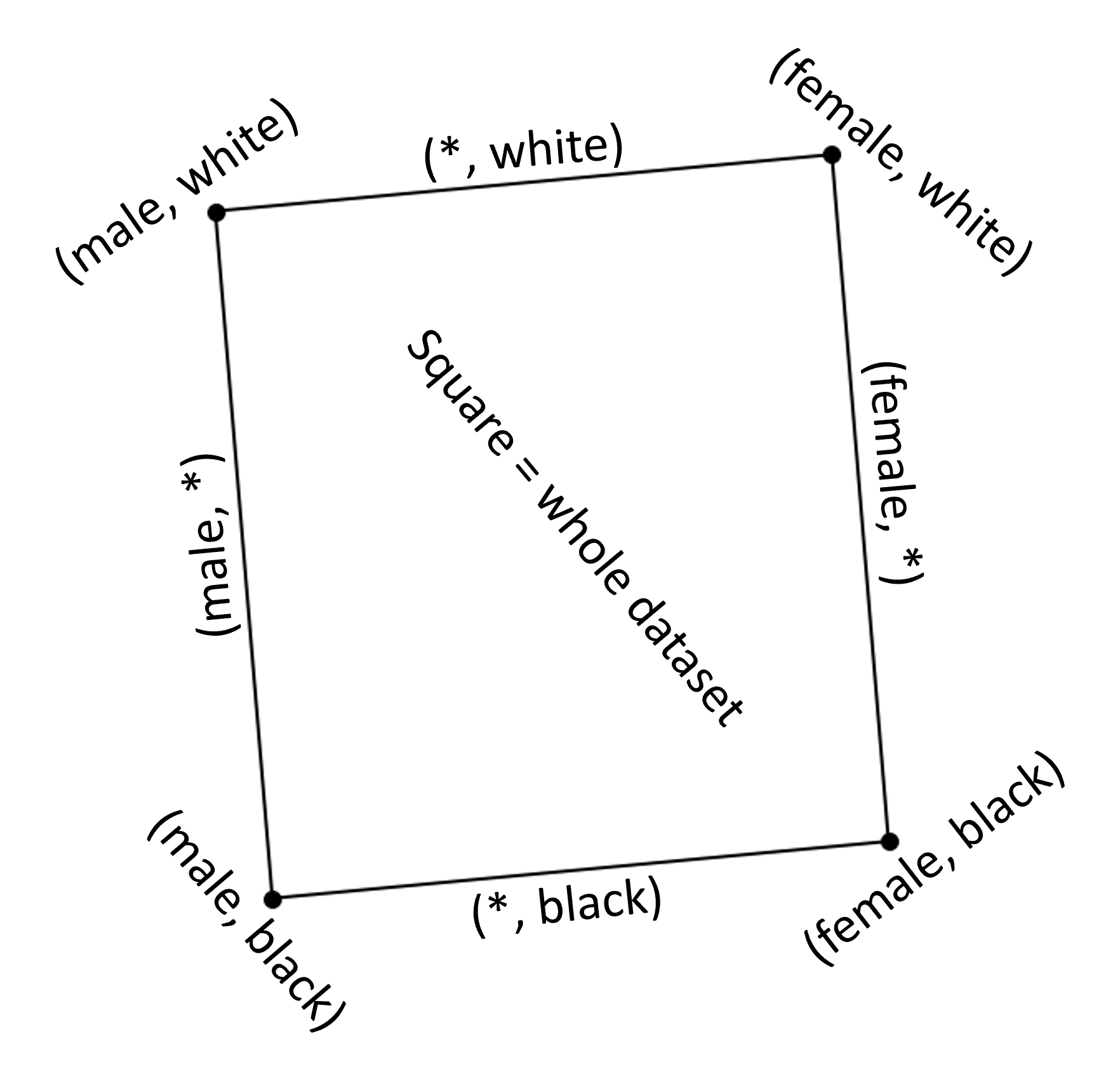}
\caption{}
\label{fig:4dcubeb}
\end{subfigure}\\
\begin{subfigure}{0.45\linewidth}
\centering
\includegraphics[width=6cm]{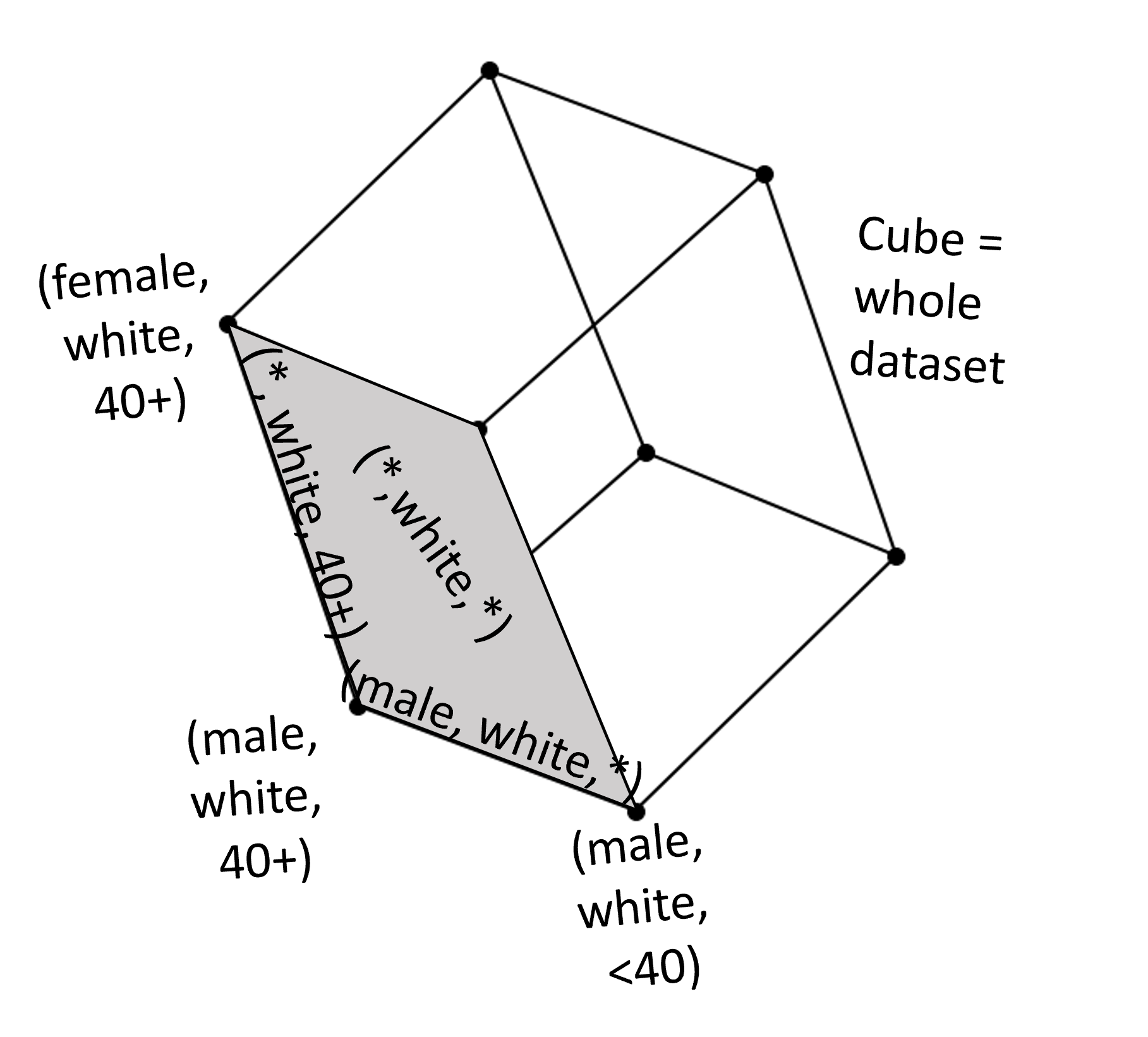}
\caption{}
\label{fig:4dcubec}
\end{subfigure}
\begin{subfigure}{0.45\linewidth}
\centering
\includegraphics[width=6.25cm]{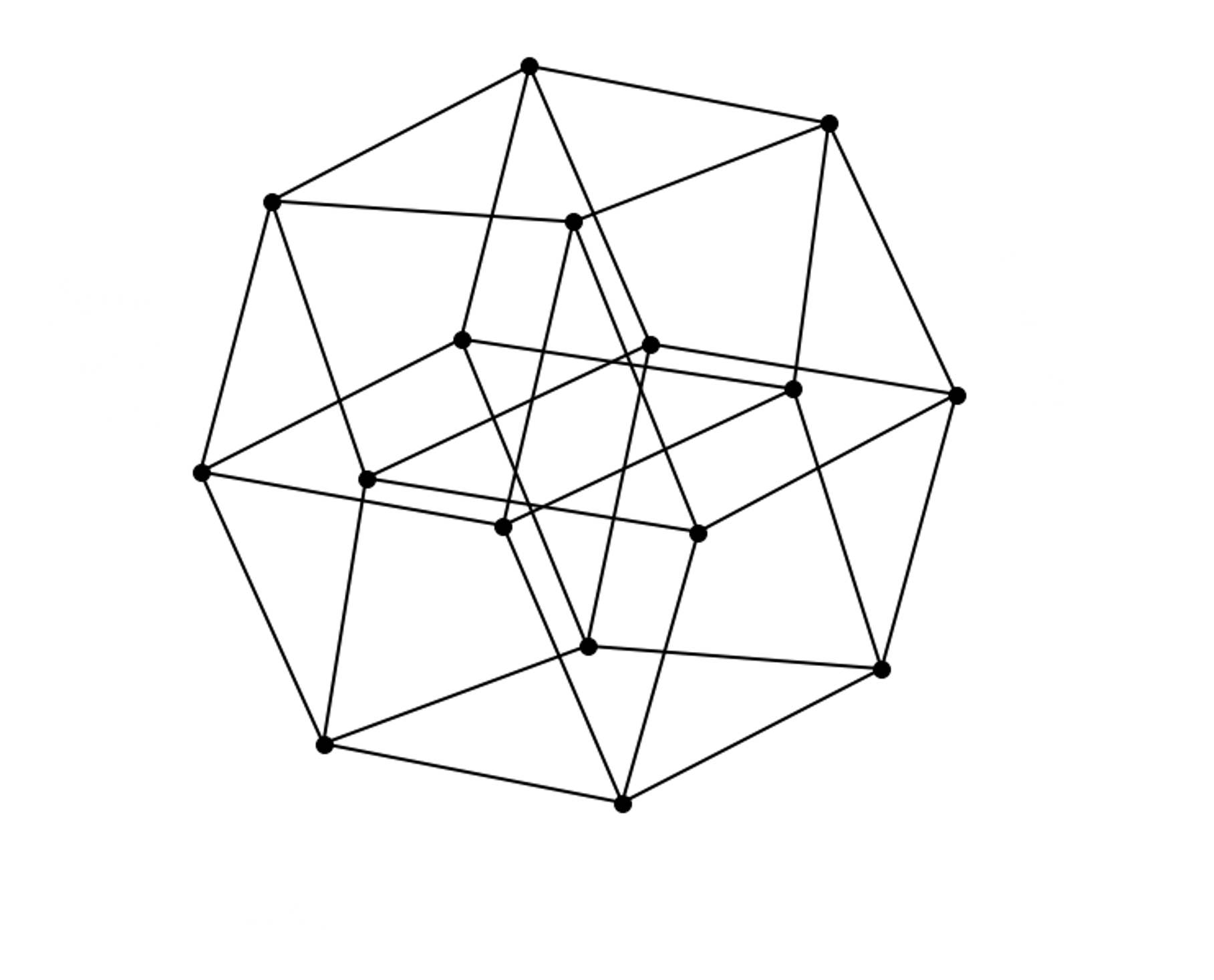}
\caption{}
\label{fig:hypercube4d}
\end{subfigure}
\caption{\textbf{Hypercubes}. \textbf{(a) 1D hypercube.} Assuming we have only one protected attribute (gender) with subgroups male/female. \textbf{(b) 2D hypercube.} Assuming there are two protected attributes (gender and ethnicity), with subgroups male/female and white/black respectively. \textbf{(c) 3D hypercube.} Assuming there are three protected attributes (gender, ethnicity and age), with subgroups male/female, white/black and <40/40+ respectively. \textbf{(d) 4D hypercube.} Projection of 4D hypercube.}
\label{fig:test1}
\end{figure}

We are given a binary classification dataset $\mathcal{D}$ with $N$ instances along with $M$ protected attributes (named $p_1, ..., p_M$). We assume all $M$ protected attributes are binary with groups labelled as 0 and 1. Please note that our analysis can be easily extended beyond the binary case, however, we choose to have binary protected attributes so that the results have a simple geometric interpretation. The first insight of our framework is that we can see our data as living on a $M$ dimensional hypercube. \\ 

A hypercube is a geometrical figure that can be extended to an arbitrary number of dimensions. A 1-dimensional hypercube is a line (Fig. \ref{fig:4dcubea}), a 2-dimensional hypercube is a square (Fig. \ref{fig:4dcubeb}), and a 3-dimensional hypercube is a cube (Fig. \ref{fig:4dcubec}). Each hypercube is made of two copies of the previous one linked together. Higher dimensional hypercubes are difficult to visualise, but share the same property (see Fig. \ref{fig:hypercube4d} for an example of a projection of a 4-dimensional hypercube). \\ 

First, we would like to explain the analogy between intersectional subgroups and hypercubes. If we consider a dataset with only one protected attribute, our dataset can be seen as belonging to a line, with each vertex being one of the subgroups (Fig. \ref{fig:4dcubea}). For example, if we consider gender as our protected attribute, then one vertex will be “male” and the other vertex will be “female”. If instead we consider two protected attributes, such as gender and ethnicity, our data will lie on a square (Fig. \ref{fig:4dcubeb}). In this case, each vertex will be the intersection of gender and ethnicity (e.g., white male, black male, white female and black female), and each line will be an individual attribute. Adding a third protected attribute, such as age, will increase the dimensionality again, and make our geometrical structure into a cube (Fig. \ref{fig:4dcubec}). In this case a single protected attribute will corresponds to the face of the cube, while the intersection of two attributes will correspond to an edge. The vertices of the cube will be the intersection of all three protected attributes. \\ 

The idea is that each protected attribute creates a division in the data, which can be seen as adding a dimension to the hypercube. The number of protected attributes therefore defines the dimensionality of the main-hypercube, which contains the whole dataset. The vertices of the hypercube, instead, comprise the intersection of all the protected attributes (e.g., $p_1=0, ..., p_M=0$). On the other hand, each edge contains the intersection of all protected attributes but one. We denote the unspecified protected attribute with the star symbol (*). More generally, if $K$ of the $M$ protected attributes are unspecified (denoted by *), the associated hypercube will be a $K$-dimensional hypercube embedded in the main-hypercube. Each hypercube will then have an associated vector $\mathbf{x} \in \{0,1,*\}^M$, which we refer to as it's vectorial specification. \\ 

We will often speak of the hypercubes as separated by levels. The level refers to the number of stars in the vectorial specification, and encodes the dimensionality of the embedded hypercube. There are $M+1$ possible levels $(0,...,M)$. Level 0 (the lowest level) corresponds to the vertices, which are the deepest intersectional groups and have dimensionality 0. Level $M$ (the highest level) consists only of the main-hypercube, which contains the whole dataset and has dimensionality $M$. The total number of hypercubes (intersectional groups), $H_{tot}$, can be computed by summing the number of hypercubes $H_K$ at each level $K$, with $K\in\{0,\hdots,M\}$:
\begin{equation}
   H_{tot} = \sum_{K=0}^M H_K= \sum_{K=0}^M \binom{M}{K}2^{M-K} = 3^M
\end{equation}

For each vertex $\mathbf{v}$ of the hypercube, we can compute the number of data points belonging to that specific subgroup. We denote this number by $N(\mathbf{v})$. We can similarly compute the number of data points with positive label that fall in the same intersection (vertex), $N^{(1)}(\mathbf{v})$. From these we can compute the success rate for any vertex as

\begin{equation}
SR(\mathbf{v}) = \frac{N^{(1)}(\mathbf{v})}{N(\mathbf{v})}
\end{equation}

Note that it only takes one run through the dataset to compute all success rates at vertex level. We do so by updating a dictionary with keys being all possible vertices. For each of the $N$ instances, we are comparing arrays of length $M$, so the total time complexity is $O(N \times M)$. In the next section, we explain how this can be used to compute the success rates for all hypercubes by propagating the computation upwards.

\section{Propagation Algorithm}
We extend the definitions of $N$ and $N^{(1)}$ to work on any hypercube $\mathbf{x} \in \{0,1,*\}^M$. The definition is the same as the one for vertices, $N(\mathbf{x})$ and $N^{(1)}(\mathbf{x})$ are respectively the number of datapoints and successful datapoints on a given structure. From these we can compute the success rate of a hypercube as

\begin{equation}
SR(\mathbf{x}) = \frac{N^{(1)}(\mathbf{x})}{N(\mathbf{x})}
\end{equation}

If $\mathbf{x}$ is fully unspecified, the associated structure is the main-hypercube, in this case we use special notation $SR(\mathbf{x})=SR_{tot}$. Let $\mathbf{x}$ be a hypercube which is not a vertex. Then by definition $\mathbf{x}$ contains at least one star. Pick the first star in $\mathbf{x}$ and consider the two sub-hypercubes contained by setting that star to 0 or 1, which we denote as $\mathbf{x}^{(0)}$ and $\mathbf{x}^{(1)}$ respectively. $\mathbf{x}^{(0)}$ and $\mathbf{x}^{(1)}$ are a partition of $\mathbf{x}$, meaning that they are disjoint and their union is $\mathbf{x}$. So the following formulas immediately follow
\begin{equation}
N(\mathbf{x})=N(\mathbf{x}^{(0)})+N(\mathbf{x}^{(1)})
\end{equation}
and
\begin{equation}
N^{(1)}(\mathbf{x})=N^{(1)}(\mathbf{x}^{(0)})+N^{(1)}(\mathbf{x}^{(1)})
\end{equation}
Using these, we can also compute the success rate for hypercube $\mathbf{x}$, as $SR(\mathbf{x})=N^{(1)}(\mathbf{x})/N(\mathbf{x})$. With equations 4 and 5, we are in a position to devise a dynamic programming algorithm for computing all the success rates.\\

To help in visualizing the propagation process, we think of each possible split as a branching on a network graph (See Fig. \ref{fig:bintree}). We refer to the resulting graph $G$ as the hypercube graph. It's nodes consist of all possible hypercubes, and the edges encode how these hypercubes split into lower dimensional ones. It is useful to separate the nodes of the graph by levels when visualising it. If we have $M$ protected attributes, the resulting graph $G$ has $3^M$ nodes in total, with $H_K=\binom{M}{K}2^{M-K}$ of them at level $K$.\\

\begin{figure}
    \centering
\includegraphics[width=8cm]{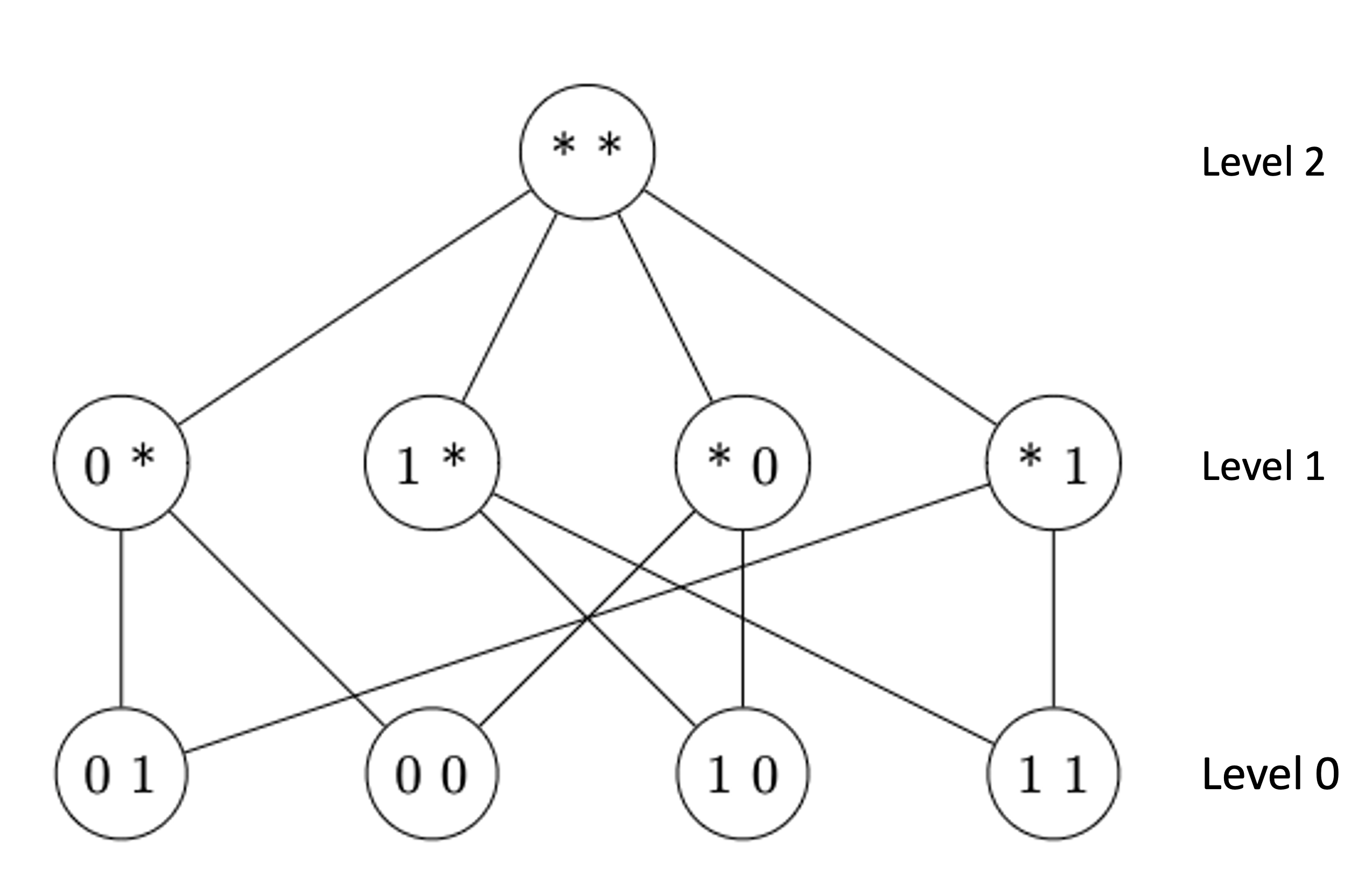}
\caption{\textbf{Hypercube graph}. Example of a hypercube graph $G$ with $M=2$ protected attributes. The nodes of the graph are all hypercubes. The edges represent the possible splits.}
\label{fig:bintree}
\end{figure}

\textbf{Algorithm 1.} We start by computing $N(\mathbf{v})$ and $N^{(1)}(\mathbf{v})$ for all vertices $\mathbf{v}$. If there is no data on a given vertex, we store a value of 0 for both of these. Then, \textbf{while} there exists a hypercube with uncomputed values. Since this hypercube is not a vertex, it's vector must contain a star. Pick the first star within the vector and let $\mathbf{x}^{(0)}$ and $\mathbf{x}^{(1)}$ be the hypercubes obtained by setting the star to 0, 1 respectively. We recursively compute the result of equations 4 and 5. Use these to store the success rate if needed. \textbf{Finish}. See Algorithm \ref{alg:cap} for pseudo-code.\\

\begin{algorithm}
\caption{$SR$ propagation}\label{alg:cap}
\begin{algorithmic}[1]
\State \textbf{Initialize} $to\_compute$ as a set containing all $3^M$ hypercube vectors.
\State \textbf{Initialize} $N, N^{(1)}$ as dictionaries with keys being all elements of $to\_compute$.
\State Loop once through dataset updating $N, N^{(1)}$ at vertex level.
\State Remove all vertex vectors from $to\_compute$.
\While {$to\_compute$ non empty}
\State Pick first element $\mathbf{x}$ of $to\_compute$.
\State Find first star in $\mathbf{x}$ and decompose into $\mathbf{x}^{(0)}$ and $\mathbf{x}^{(1)}$.
\State Recursively compute values of $N, N^{(1)}$ using equations 4 and 5.
\State Remove $\mathbf{x}$ from $to\_compute$.
\EndWhile
\end{algorithmic}
\end{algorithm}
In Appendix \ref{appendix:complexity}, we include a study of the complexity of the propagation algorithm, as opposed to the "brute force" approach. The results of the analysis are summarised in Table \ref{tab:complexitytable}. The main takeaway is that for large $N$ and $3^M$, the difference between the multiplicative ("brute force") and additive ("propagation") complexities becomes considerable.
\begin{table}
    \centering
\begin{tabular}{|c|c|} 
\hline
  \textbf{Algorithm} & \textbf{Complexity}\\ 
\hline
  Brute Force Approach & $O(N\times 3^M)$\\ 
  \hline
  $SR$ Propagation & $O(N+3^M)$\\ 
  \hline
\end{tabular}
\caption{\textbf{Complexity Table.} Brute force approach consists of filtering the dataset for each intersectional subgroup, and computing success rate on filtered data. $SR$ Propagation first computes success rates at vertex level, and then propagates them upwards (Algorithm \ref{alg:cap}).}
\label{tab:complexitytable}
\end{table}

\section{Fairness Propagation}
\subsection{Equality of Outcome}
We are given a dataset $\mathcal{D}$ of size $N$ with $M$ binary protected attributes $p_1, ..., p_M$ and one binary label $\mathbf{y}$. Define the minimum and maximum success at level $K$ as
\begin{equation}
SR_{min}(K) = \min_{\mathbf{x} \in \textrm{Level } K} SR(\mathbf{x})
\end{equation}
\begin{equation}
SR_{max}(K) = \max_{\mathbf{x} \in \textrm{Level } K} SR(\mathbf{x})
\end{equation}
In Theorem 1, we will show that the minimum and maximum success rates must narrow down as we go up the levels. This result is similar to (Theorem IV.1. \cite{Foulds2020}), although derived in a different framework.\\
\begin{theorem}
$SR_{min}(K)$ is an increasing function of $K$ and $SR_{max}(K)$ is a decreasing function of $K$. In language, we refer to this result as: "success rates narrow as we go up the levels".
\end{theorem}
\begin{proof}

Let $\mathbf{x}$ be a non-vertex hypercube (of dimension $K$) and let $\mathbf{x}_i^{(0)}$ and $\mathbf{x}_i^{(1)}$ be a partition of it using the $i^{th}$ star. Recall that we can compute the success rate as
\begin{equation}
SR(\mathbf{x}) = \frac{N^{(1)}(\mathbf{x})}{N(\mathbf{x})}
\end{equation}
Using equation 5, this is equivalent to
\begin{equation}
SR(\mathbf{x}) = \frac{N^{(1)}(\mathbf{x}_i^{(0)})+N^{(1)}(\mathbf{x}_i^{(1)})}{N(\mathbf{x})}
\end{equation}
Which can be rewritten as
\begin{equation}
SR(\mathbf{x}) = \frac{N(\mathbf{x}_i^{(0)})}{N(\mathbf{x})}SR(\mathbf{x}_i^{(0)}) + \frac{N(\mathbf{x}_i^{(1)})}{N(\mathbf{x})}SR(\mathbf{x}_i^{(1)})
\end{equation}
Note that equation 10 is a weighed average of the success rates of the two sub-hypercubes. The weights are given by the respective proportions of datapoints coming from each of the two sub-hypercubes. A weighed average lies between the minimum and maximum values we are averaging over. So by looking at the $i^{th}$ split, we obtain the following bound on the success rate $SR(\mathbf{x})$.
\begin{equation}
\min_j SR(\mathbf{x}_i^{(j)}) \leq SR(\mathbf{x}) \leq \max_j SR(\mathbf{x}_i^{(j)})
\end{equation}
If we instead choose to look at all $K$ ways of splitting the hypercube $\mathbf{x}$, we obtain a better bound.
\begin{equation}
\max_i \min_j SR(\mathbf{x}_i^{(j)}) \leq SR(\mathbf{x}) \leq \min_i \max_j SR(\mathbf{x}_i^{(j)})
\end{equation}
Using equation 12, we conclude the proof of the theorem.
\begin{equation}
SR_{min}(K) = \min_{\mathbf{x} \in \textrm{Level } K} SR(\mathbf{x}) \geq \min_{\mathbf{x} \in \textrm{Level } K} \max_i \min_j SR(\mathbf{x}_i^{(j)})
\geq \min_{\mathbf{x} \in \textrm{Level } K-1} SR(\mathbf{x})=SR_{min}(K-1)
\end{equation}
\begin{equation}
SR_{max}(K) = \max_{\mathbf{x} \in \textrm{Level } K} SR(\mathbf{x}) \leq \max_{\mathbf{x} \in \textrm{Level } K} \min_i \max_j SR(\mathbf{x}_i^{(j)})
\leq \max_{\mathbf{x} \in \textrm{Level } K-1} SR(\mathbf{x})=SR_{max}(K-1)
\end{equation}
\end{proof}

Please note that we didn't need to use all $K$ ways of splitting (equation 12) to obtain the above result, picking any one of them (equation 11) would have sufficed. We choose to include all ways of splitting within the formula, because it highlights an important property of how the success rates relate across levels. That is that each hypercube at level $K$ is constrained in $K$ different ways.\\

If we further define the worst case disparate impact \cite{Feldman2014} and the worst case statistical parity \cite{Feldman2014} at level $K$ as
\begin{equation}
DI(K) = \frac{SR_{min}(K)}{SR_{max}(K)}
\end{equation}
and
\begin{equation}
SP(K) = SR_{max}(K)-SR_{min}(K)
\end{equation}
We can prove the following corollary:
\begin{corollary}
$DI(K)$ increases with increasing $K$, and $SP(K)$ decreases with increasing $K$. In language, we refer to these results as: "Equality of Outcome fairness propagates up the levels".
\end{corollary}
\begin{proof}
As follows from Theorem 1:
\begin{equation}
DI(K) = \frac{SR_{min}(K)}{SR_{max}(K)} \geq \frac{SR_{min}(K-1)}{SR_{max}(K-1)} = DI(K-1)
\end{equation}
and
\begin{equation}
SP(K) = SR_{max}(K)-SR_{min}(K) \leq SR_{max}(K-1)-SR_{min}(K-1) = SP(K-1)
\end{equation}
\end{proof}

\subsection{Equality of Opportunity}
In Appendix \ref{appendix:eoo} we extend the above analysis to the Equality of Opportunity setting. We do so by adapting Theorem 1, to show that other fairness measures also narrow as we progress up the levels such as: accuracy, precision, true positive rates, false positive rates, true negative rates, false negative rates \cite{hardt2016equality}. In language, these results can be summarized as follows: "\textit{Equality of Opportunity fairness also propagates up the levels}".

\section{Intersectional Statistical Parity}
In the previous sections, we used a minimum and maximum approach, to get worst case scenario proofs about the evolution of fairness metrics as we progress up the levels. In this section, we use a statistical interpretation of our framework. First, it will be useful to introduce some notation from probability theory. We let $A_1,\hdots,A_M$ denote the random variables associated with our $M$ protected attributes and $Y$ be the random variable associated to the label. Intersectional Statistical Parity (ISP) holds when there is independence between the protected attributes (including intersections of an arbitrary number of them) and the output label (\cite{dwork2012fairness, agarwal2018reductions, Kearns2018, Foulds2020}). In equations,
\begin{equation}
    P(Y|A_{i_1},\hdots,A_{i_p})=P(Y)
\end{equation}
In the presence of ISP, every single hypercube has precisely the same probability of success $P(Y=1)=p_{tot}$. For each hypercube $\mathbf{x}$, the success rate $SR(\mathbf{x})$ can therefore be seen as the sample mean of $N(\mathbf{x})$ draws of a $Bernoulli(p_{tot})$ random variable. Hence, the mean and variance are given by:
\begin{equation}
\mathbb{E}[SR(\mathbf{x})]=p_{tot} 
\end{equation}
and
\begin{equation}
Var [SR(\mathbf{x})]=\frac{p_{tot}(1-p_{tot})}{N(\mathbf{x})}
\end{equation}



We first consider the simplified case where each subgroup on a given level has an equal number of instances. In this case, the number of datapoints  for each hypercube at level $K$ is exactly $N_K=N2^{K-M}$ (Appendix \ref{appendix:avgnum}). In this scenario, all success rates $SR(\mathbf{x})$ have the same distribution, with mean and variance given by equations 20 and 21 respectively and $N(\mathbf{x}) = N_K$. The theoretical value for the variance (at level $K$) is then:
\begin{equation}
Var_{ISP}(K) = \frac{p_{tot}(1-p_{tot})}{N_K} =\frac{2^Mp_{tot}(1-p_{tot})}{N}\times 2^{-K} 
\end{equation}
Note that the variance decreases exponentially with increasing levels. By setting $\alpha = \frac{2^Mp_{tot}(1-p_{tot})}{N}$ and transforming to the log domain, we get a simpler formulation
\begin{equation}
    \log Var_{ISP}(K) = \log \alpha - K \log 2
\end{equation}
This means that, if ISP holds, the log of the variance at level $K$ decreases linearly with increasing $K$. The slope of the linear curve is $-\log 2$ and the intercept is $\log \alpha$.\\

Supposing we are now in possession of data, obtained from an ISP classifier. We can use the empirical success rates at each level to compute an estimate of the variance as:
\begin{equation}
    \frac{1}{H_K}\sum_{\mathbf{x}\in\textrm{Level }K}(SR(\mathbf{x})-\overline{SR}(K))^2
\end{equation}
The above formula is known to be an estimator of the theoretical variance, under the assumption of independent samples. However, on any level above Level 0, the samples are technically not independent. That is because different hypercubes share vertices and we find that especially at very high levels, the approximation of assuming independence breaks down. We mitigate this issue by using sub-sampling. By sub-sampling our dataset several times, we solve two problems at once. The first is that we ensure the dataset is balanced, which is a crucial assumption in the derivations. The second is that we produce more samples, that are less correlated overall. We call the empirical estimate of the variance obtained using sub-sampling $Var(K)$. If we subsample our data $n_{sub}$ times, $Var(K)$ can be defined mathematically as
\begin{equation}
    Var(K) = \frac{1}{n_{sub}\times H_K}\sum_{i=1}^{n_{sub}}\sum_{\mathbf{x}\in\textrm{Level }K}(SR_i(\mathbf{x})-\overline{SR}(K))^2
\end{equation}



What happens if ISP does not hold? In this case, the set of success rates $SR(\mathbf{x})$ at each level is made up of copies of random variables with different means. Since the means are different, we expect the variance to be greater than the theoretically computed minimum $Var_{ISP}(K)$. We can then use this theoretical value for the variance as a benchmark, to define a measure of intersectional bias at each level. We call this family of metrics $VarRatio(K)$, and define it mathematically as
\begin{equation}
    VarRatio(K)=\frac{Var(K)}{Var_{ISP}(K)}
\end{equation}
In this paper, we restrict our focus to the metric obtained at vertex level, for $K=0$. Once again using $\alpha = \frac{2^Mp_{tot}(1-p_{tot})}{n_{sub}}$ to simplify the expression, the metric can be written as:
\begin{equation}
    VarRatio(0)=\frac{Var(0)}{\alpha}
\end{equation}
There are three cases that can arise:
\begin{enumerate}
    \item $VarRatio(0)\approx1$ indicating the classifier satisfies (or is close to satisfying) ISP.
    \item $VarRatio(0)>1$ indicating the classifier contains some intersectional bias.
    \item $VarRatio(0)<1$ hinting that the classifier was trained or post-processed with some bias mitigation procedures.
\end{enumerate}

Please note that if the data is unbalanced and we do not use a sub-sampling approach, $Var_{ISP}(K)$ still serves as an approximate lower bound for the expected variance at Level $K$ (Appendix \ref{appendix:varformula}).\\

In conclusion, we showed that, under perfect fairness, the variance of the empirical success rates decreases exponentially across levels. This means that the log-variance against level curve is linear, and this can be used as a test for the ISP criterion. These theoretical findings helped us define a family of metrics $VarRatio(K)$ which capture intersectional bias, and because of how they are constructed, we expect them to be less affected by small sample variance than other approaches within the literature. 
\section{Experiments}
\subsection{Synthetic Data}
\begin{figure}[h]
\centering
\includegraphics[width=7cm]{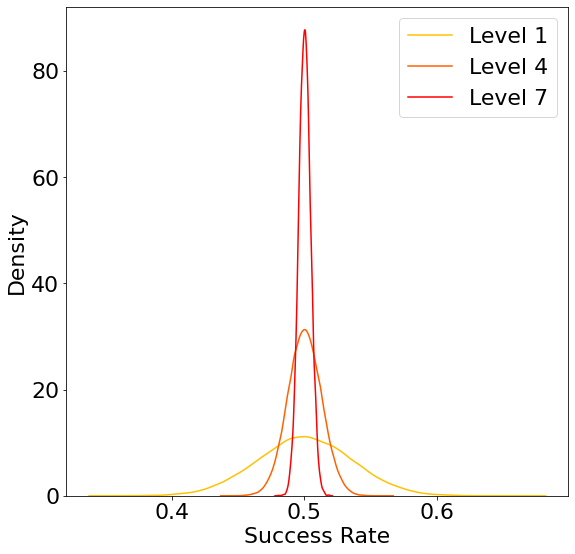}
\caption{\textbf{Success Rate Distributions}. Estimated empirical distributions of the success rates over all the hypercubes at levels 1, 4 and 7 for Experiment 1. The distributions get narrower around $SR_{tot}=0.5$ as we travel up the levels.}
\label{fig:dist}
\end{figure}



In order to verify our theoretical results, we run experiments with synthetic data. We create a function to generate random datasets, with $M=10$ protected attributes, and one label. Each of the $2^{10}=1024$ smallest intersectional groups (vertices), $\mathbf{v}$, contains $N(\mathbf{v})=200R$ instances, where $R$ is a randomly selected integer between 1 and 10. Each vertex also has an associated probability of success $p(\mathbf{v})$, which we use to sample its binary label data from a $Bernoulli(p(\mathbf{v}))$ distribution. We sub-sample the dataset $n_{repeats}=20$ times so as to have all vertices contain precisely $n_{sub}=100$ datapoints. The sub-sampling is crucial for two reasons. The first is that it allows us to generate a larger number of samples, with minimal amounts of correlation. The second is that it ensures all hypercubes have equal number of datapoints. Both of these conditions were assumptions in our theoretical derivations.
\\

In this first example, we will consider a fair dataset, where $p(\mathbf{v})=0.5$ for all vertices. Since the output label is here independent of all the protected attributes, this example satisfies Intersectional Statistical Parity. We run the Propagation Algorithm \ref{alg:cap} to compute the empirical success rates over all the possible intersectional groups, that is, all the hypercubes in our geometrical model. Fig. \ref{fig:dist} shows the distribution of the empirical success rates at different levels. As predicted by our theoretical results, the distribution gets narrower around $SR_{tot}=0.5$ as we go up the levels. Fig. \ref{fig:singlesynth1sub1} shows how the minimum and maximum values evolve as we progress up the levels. The maximum value decreases and minimum value increases, as shown in Theorem 1. Fig. \ref{fig:singlesynth1sub2} and Fig. \ref{fig:singlesynth1sub3} show plots of $Var(K)$ and $\log Var(K)$ as a function of $K$. Both curves match the theoretical results precisely, except for small numerical approximations.\\

\begin{figure}[H]
\centering
\begin{subfigure}{\linewidth}
\centering
\includegraphics[width=14.4cm]{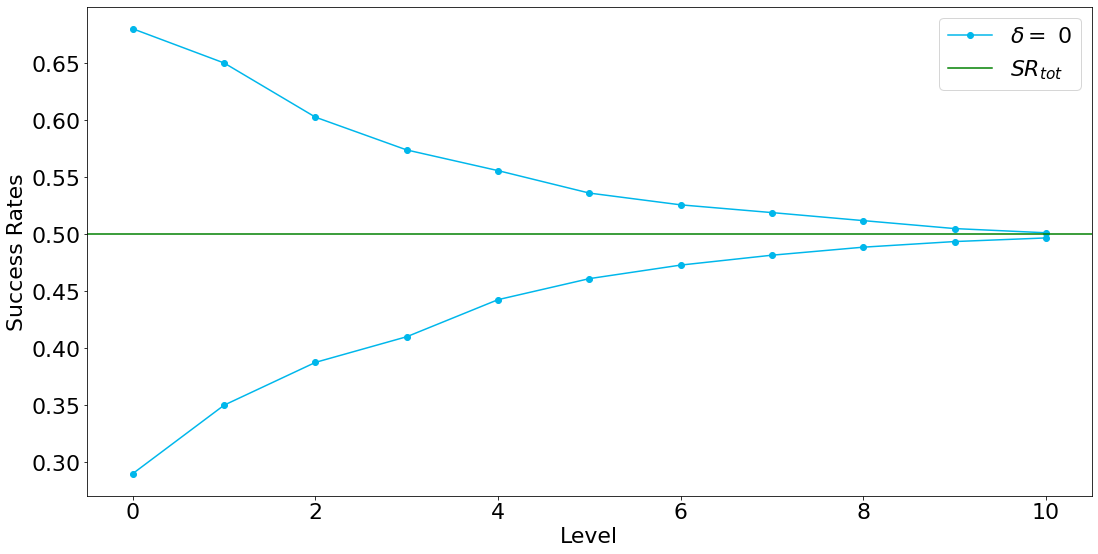}
\caption{minimum and maximum}
\label{fig:singlesynth1sub1}
\end{subfigure}
\begin{subfigure}{0.45\linewidth}
\centering
\includegraphics[width=7.3cm]{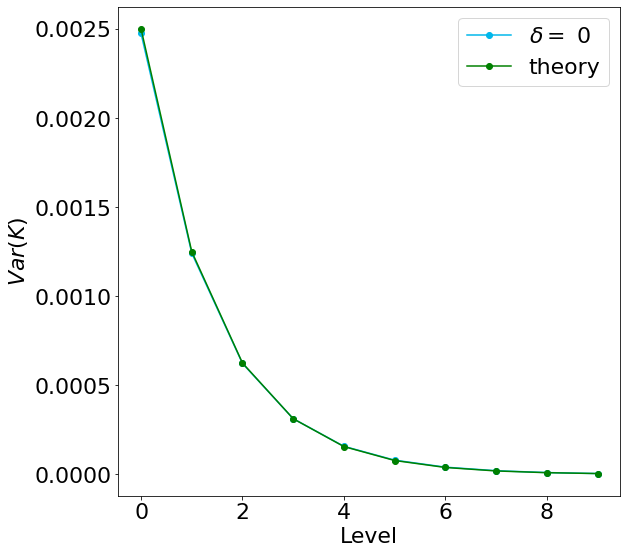}
\caption{variance}
\label{fig:singlesynth1sub2}
\end{subfigure}
\begin{subfigure}{0.45\linewidth}
\centering
\includegraphics[width=7cm]{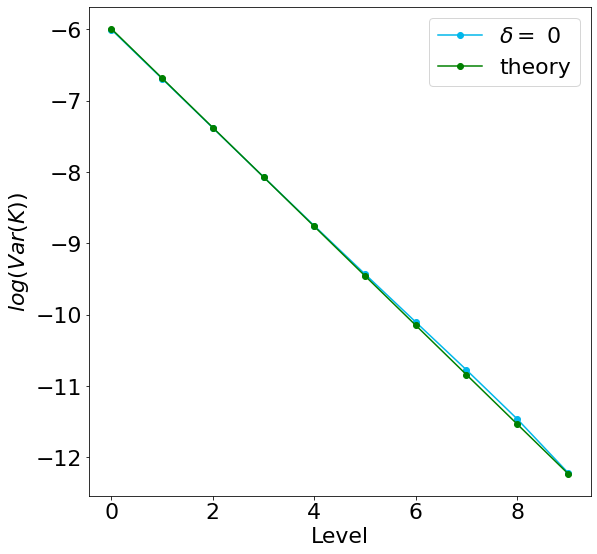}
\caption{log variance}
\label{fig:singlesynth1sub3}
\end{subfigure}
\caption{\textbf{Experiment 1}. (a) Evolution of the minimum and maximum empirical success rates across levels. As expected from the theoretical results, these values tend to 0.5. (b) Evolution of the Variance across levels. The variance decreases exponentially, as predicted by our theoretical results. (c) Evolution of the log-Variance across levels. The log-variance linearly, as predicted by our theoretical results.}
\label{fig:synth1data}
\end{figure}

In our second experiment, we want to analyse how these results change in the presence of bias. We model bias by allowing each vertex $\mathbf{v}$ to have a different probability of success $p(\mathbf{v})$.
Specifically, we select a set of $100$ vertices to have probability of success $p(\textbf{v})=0.5 - \delta$, and a separate set of $100$ vertices to have probability of success $p(\textbf{v})=0.5 + \delta$. We set our parameter $\delta$ to vary in the interval $\{0, 0.1, 0.2, 0.3, 0.4\}$. When $\delta=0$, we retrieve the case of perfect fairness from the previous example. Otherwise, ISP will not be satisfied, and indeed our dataset will get less and less fair as $\delta$ increases. The results show that, as expected, the minimum and maximum success rates narrow down as fairness increases, that is, for lower values of $\delta$ (Fig. \ref{fig:synthsub1}). We also show that, just as our theoretical results had predicted, the variance of the empirical success rates increases with higher levels of bias (Fig. \ref{fig:synthsub2}). Accordingly, the evolution of the log-variance across levels is not linear but concave for higher levels of $\delta$, another indicator that ISP is not satisfied. Finally, we calculate our metric $VarRatio(0)$ (Table \ref{tab:exp2metricstable}), and find that it increases for increasing values of $\delta$, which indicate larger bias.

\begin{figure}
\centering
\begin{subfigure}{\linewidth}
\centering
\includegraphics[width=14.4cm]{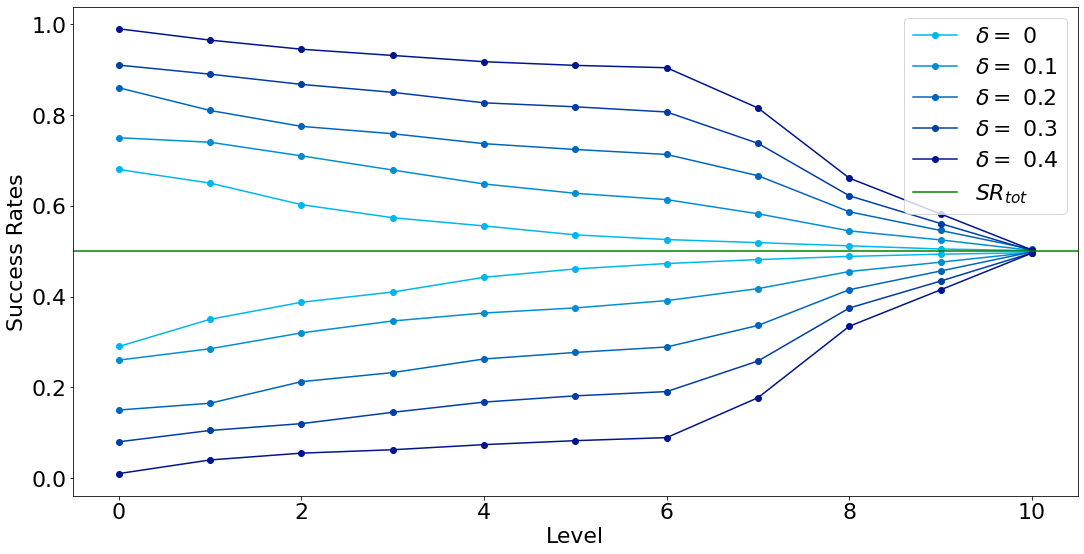}
\caption{minimum and maximum}
\label{fig:synthsub1}
\end{subfigure}
\begin{subfigure}{0.45\linewidth}
\centering
\includegraphics[width=7.3cm]{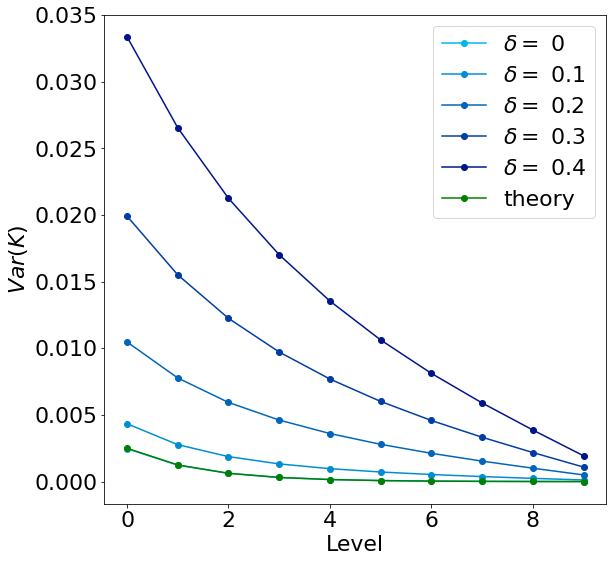}
\caption{variance}
\label{fig:synthsub2}
\end{subfigure}
\begin{subfigure}{0.45\linewidth}
\centering
\includegraphics[width=7cm]{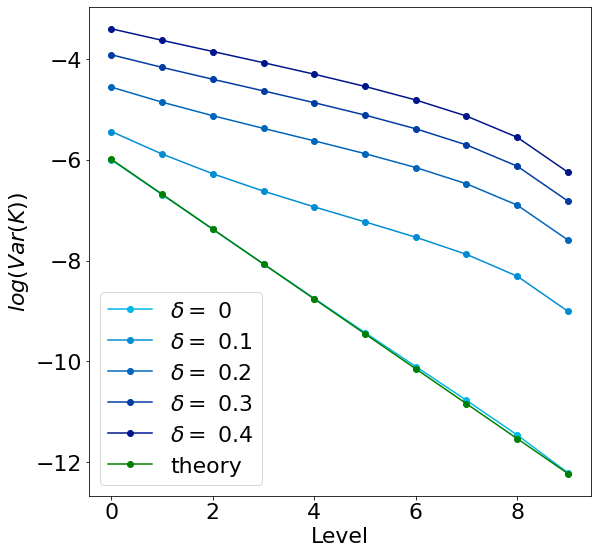}
\caption{log variance}
\label{fig:synthsub3}
\end{subfigure}
\caption{\textbf{Experiment 2}. (a) Evolution of the minimum and maximum empirical success rates across levels. As expected from the theoretical results, these values tend to 0.5. These values narrow down for smaller $\delta$, as bias decreases. (b) Evolution of the Variance across levels. The variance increases in the presence of bias. (c) Evolution of the log-Variance across levels. The log-variance curve departs from the theoretical value as we increase $\delta$.}
\label{fig:fakedata}
\end{figure}

\begin{table}
    \centering
\begin{tabular}{|c|c|c|c|c|c|c} 
\hline
  \textbf{ } & \textbf{$\delta=0$} & \textbf{$\delta=0.1$} & \textbf{$\delta=0.2$} & \textbf{$\delta=0.3$} & \textbf{$\delta=0.4$}\\ 
\hline
  Level 0  & $0.98$ & $1.73$ & $4.19$ & $7.96$ & $13.34$ \\ 
  \hline
\end{tabular}
\caption{\textbf{Metrics Table}. $VarRatio(0)$ calculated for experiment 2. }
\label{tab:exp2metricstable}
\end{table}

\subsection{Real Data Experiment}
For the real data test, we chose to use the Adult dataset. The Adult dataset is one of the most commonly used datasets in fair-AI research. That is because it has many attributes that can be considered sensitive (sex, race, age, etc) and the task is that of binary classification. The dataset consists of $N=48842$ instances with 14 features and one label feature. The usual task given on this dataset is the prediction of the binary attribute of salary being above 50K from the other features.

For the purposes of our experiments, we performed some simple prepossessing on the data. The first step is we keep only 5 of the 15 provided columns. That is $M=4$ protected attributes (sex, race, age, marital-status) and the binary label (class). We also performed some grouping on the attributes age, race and marital-status. For age, we binarize the data so that instances with age $\geq 40$ are set to 1 and all others are set to 0. For race we grouped Black, Asian-Pac-Islander and Amer-Indian-Eskimo with the Other group. For marital-status we grouped the 7 provided statuses (Married-civ-spouse, Never-married, Divorced, Separated, Widowed, Married-spouse-absent, Married-AF-spouse) into Married and Not-Married. The reason we had to perform grouping on some attributes is to match the setting of the theoretical work (all protected attributes binary). Furthermore the grouping ensures all intersectional subgroups contain sufficient data. For the average and minimum number of datapoints over all hypercubes per level, see Table \ref{tab:realnums1}.

\begin{table}[h]
\centering
\begin{tabular}{lccccc}
\toprule
{} &  Level 0 &     Level 1 &  Level 2 &  Level 3 &  Level 4 \\
\midrule
\textit{Average Number} &  3052 &     6105 &  12210 &  24421 &  48842\\
\textit{Minimum Number} &  228 &     521 & 2511 &  7080 &  48842\\
\bottomrule
\end{tabular}
\caption{\textbf{Adult Dataset}. Minimum number and average number of datapoints at each level $K=0,\hdots,4$.}
\label{tab:realnums1}
\end{table}

\begin{figure}[h]
\centering
\begin{subfigure}{\linewidth}
\centering
\includegraphics[width=14.4cm]{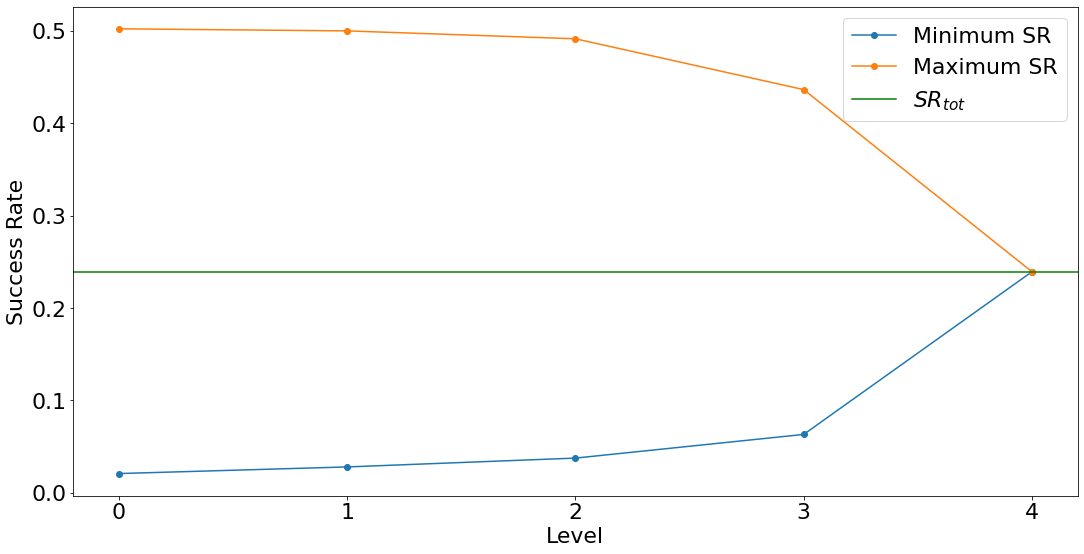}
\caption{minimum and maximum}
\label{fig:realsub1}
\end{subfigure}\\
\begin{subfigure}{0.45\linewidth}
\centering
\includegraphics[width=7.3cm]{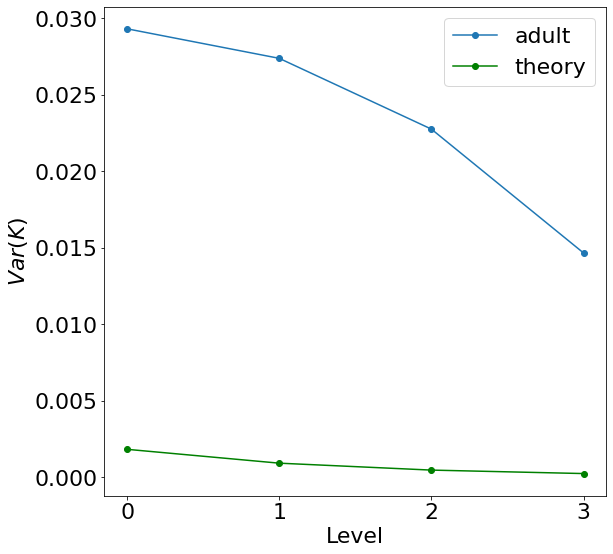}
\caption{variance}
\label{fig:realsub2}
\end{subfigure}
\begin{subfigure}{0.45\linewidth}
\centering
\includegraphics[width=7cm]{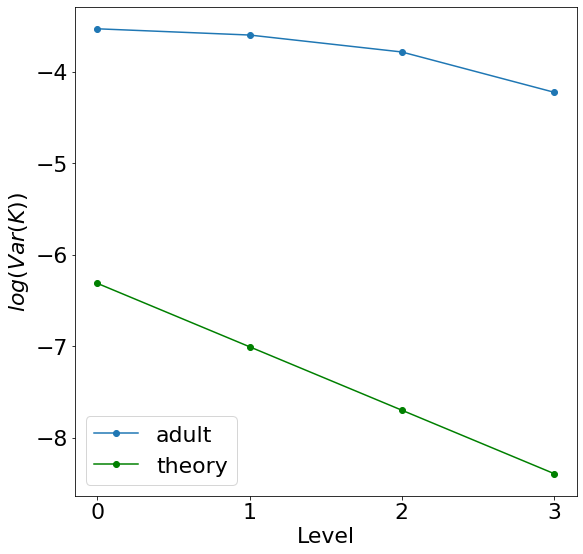}
\caption{log variance}
\label{fig:realsub3}
\end{subfigure}
\caption{\textbf{Adult Dataset Experiment}. (a) Evolution of the minimum and maximum empirical success rates across levels. (b) Evolution of the Variance across levels. (c) Evolution of the log Variance across levels. }
\label{fig:realdata}
\end{figure}

The following results are obtained by sub-sampling the dataset $n_{repeats}=20$ times so that each vertex contains exactly $n_{sub}=100$ samples. In Fig. \ref{fig:realsub1} the minimum and maximum success rates are displayed as a function of level. As we have shown, the success rates must narrow with increasing levels, although they narrow more slowly for the Adult dataset. Fig. \ref{fig:realsub2} and Fig. \ref{fig:realsub3}  show the evolution of the variance and log-variance across levels. Observe the big difference in variance between the theoretically computed values (obtained under ISP condition) and Adult dataset values. All results indicate the presence of intersectional bias within the data.

\section{Fractal Fairness}
Intersectional fairness in the context of AI typically focuses on splitting groups of people into smaller identity categories by intersecting multiple protected attributes. Although it is commendable that the issue of intersectionality is being tackled by the community, we feel that the current approaches have some limitations. For example, it is not clear how many protected attributes one should consider, which protected attributes one should consider and how to select the level of granularity for the analysis \cite{Kong2022}. \\ 

On the other hand, we sought a framework that would allow us to think of intersectional fairness \textit{as a whole}: one complex system involving all possible intersectional groups at all levels of description. We turned to geometry, and particularly fractals, for inspiration. Fractals are geometrical objects that, informally, can be broken down into smaller objects that resemble the original shape \cite{mandelbrot1982fractal}. Although our framework is not mathematically a fractal, we identified a geometrical structure where each hypercube separates into further instances of hypercubes. This provided us with a mathematical model to study a nested and recursive notion of fairness. In fractals, patterns at a small scale repeat at a larger scale. Similarly, our framework asks that our data is fair for any possible subgroup of the population, no matter which or how many attributes are intersected. \\ 

Furthermore, the geometrical nature of our setting allows us to examine how fairness propagates and scales. We found that fairness necessarily propagates up the levels, while bias propagates down. This suggests that we should, first and foremost, think of fairness in a bottom-up and local way. It also indicates that we cannot think of individual parts of a system in isolation from the whole. We propose that the trustworthiness of an AI system cannot thus be enforced in a top-down way, or by separating individual components from the larger system. Instead, we argue for a bottom-up emergent and holistic approach to trustworthy AI.  

\section{Discussion}
In this paper, we introduced a geometrical setting for studying the fairness properties of all possible intersectional subgroups together, in a unified framework. There are multiple advantages to this setting. Firstly, it allows for a quicker computation of all the success rates (or other metrics if needed), using a dynamic programming approach. Secondly, it reveals the inter-connectivity between the fairness properties of groups at different levels. In particular, we prove that success rates must narrow as we go up the levels. We use that to prove worst case disparate impact increases with increasing levels, and worst case statistical parity decreases with increasing levels. We summarise these results as \textit{"Equality of Outcome fairness propagates up the levels"}. We then extend our results to the Equality of Opportunity setting. In particular, we show that accuracy, precision, true positive rates, false positive rates, true negative rates and false negative rates all narrow down as we progress up the levels. We summarise these results as \textit{"Equality of Opportunity fairness also propagates up the levels"}. In the future, it might be interesting to derive a mathematical description of all metrics for which this method of proof applies.\\

Furthermore, we suggest that the variance of the empirically computed success rates on a given level can be used as a fairness measure. We prove that under perfect fairness (Intersectional Statistical Parity), the variance follows a simple exponential scaling law. Using this theoretical value as benchmark, we define a family of metrics which capture the intersectional bias on each level. Future work will definitely focus on examining these metrics further and studying the behaviour of the variance and log-variance curves more closely in different settings. One of the hopes is that, the shape of the log-variance curve can help in detecting but also in locating the source of bias. For instance if bias was injected into the system using combinations of two protected attributes (Level $M-2$), we might see a bend in the log-variance curve at around those values of $K$. This could help answer questions related to the sources of bias and identifying gerrymandering. Alternatively, the variance scaling law could be used to devise similarly structured statistical tests.  \\

One of the advantages of our framework is that it allows us to analyse the whole system at once. Bias is not exclusively measured for given groups at a prescribed intersectional level, but instead its evolution is examined across all levels. However, most of our analyses aggregate over groups. This issue could be mitigated, for example, by using our framework in conjunction with other bias measures. In the future, it would be interesting to employ our propagation algorithm for more detailed modelling. For example, we could analyse how the success rates for two specific attributes interact and spread across levels. Otherwise, we could study the evolution of the fairness metrics for one specific attribute (e.g. gender) across levels or even specific paths within the hypercube. This would provide us with further insight into how bias affects that specific attribute. More broadly, our way of framing the problem of intersectionality can reveal the interconnections among intersectional subgroups, which could be key to developing appropriate metrics and mitigation techniques.

\bibliographystyle{ACM-Reference-Format}
\bibliography{bib.bib}

\appendix

\section{Complexity of Success Rate Propagation Algorithm}
\label{appendix:complexity}
 We first study the complexity of the "brute force" approach to computing all success rates. Suppose we filter the dataset by the appropriate intersectional group and then compute the success rates on the filtered data for each of the $3^M$ intersections. For each of the $3^M$ intersections, we are looping over $N$ instances and for each instance we must compare arrays of length $M$. So this method would have time complexity $O(N\times M\times 3^M)$.\\
 
As we have seen, we can compute the quantities at vertex-level in $O(N\times M)$ time. What is the time complexity of the propagation algorithm? To compute the complexity of the propagation algorithm, we consider the hypercube graph $G$ (see Fig. \ref{fig:bintree}). Each hypercube at level $K$ has $K$ stars, so it is connected with $2K$ hypercubes one level below. Recall that level $K$ consists of $H_K=\binom{M}{K}2^{M-K}$ hypercubes. So the total number of edges $E(G)$ in $G$ can be written as a sum over levels $K\in \{1,\hdots,M\}$.
\begin{equation}
E(G) = \sum_{K=1}^{M}2K\binom{M}{K}2^{M-K} \leq \sum_{K=0}^{M}2M\binom{M}{K}2^{M-K} = 2M\times 3^{M}
\end{equation}
Since our algorithm functions with recursion on the graph $G$, by reusing values that have already been computed, it would never cross the same edge twice. So the number of edges in $G$ is a strict upper bound on the time taken by the algorithm. So the time complexity of the dynamic programming approach is upper bounded by $O(N\times M+M\times 3^M)$. Assuming $M$ is small compared to $N$ and $3^M$ (and therefore modelled as a constant), the complexities of both approaches are summarized in Table \ref{tab:complexitytable}.

\section{Extension to Equality of Opportunity}
\label{appendix:eoo}

\subsection{Accuracy}
Until now we have only been working in an Equality of Outcome fairness framework. Meaning that we are seeking to equalize the outcomes of our model for different subgroups of the population. There are also Equality of Opportunity fairness notions, that seek to equalize the performance of our model on different subgroups. There is a quick way to translate our findings from an Equality of Outcome framework to an Equality of Opportunity framework. Suppose that we are given the true labels as well as the predicted labels. Name these vectors $\mathbf{y}_{true}$ and $\mathbf{y}_{pred}$ respectively. Consider the new vector 
\begin{equation}
\mathbf{y}_{correct} = \mathbf{y}_{true} \land \mathbf{y}_{pred}
\end{equation}
Suppose we use Theorem 1 with $\mathbf{y}_{correct}$ instead of $\mathbf{y}_{pred}$. We can show that the success rate of $\mathbf{y}_{correct}$ narrows as we go up the levels. However, the success rate of $\mathbf{y}_{correct}$ is precisely the accuracy of $\mathbf{y}_{true}$, $\mathbf{y}_{pred}$. So we have shown that accuracy also narrows as we go up the levels. If we define the worst case accuracy ratio and worst case accuracy difference at level $K$ as
\begin{equation}
Acc_{ratio}(K) = \frac{Acc_{min}(K)}{Acc_{max}(K)}
\end{equation}
\begin{equation}
Acc_{diff}(K) = Acc_{max}(K)-Acc_{min}(K)
\end{equation}
By using Corollary 2, we can show that $Acc_{ratio}(K)$ increases and $Acc_{diff}(K)$ decreases with increasing $K$.

\subsection{Other Metrics}
In the previous section, we showed how we can extend the original result about success rates to the analogous result about accuracy. The method we used to do so can be framed more generally. In particular, we can show analogous results for any metric that can be written as a success rate of a binary vector $\mathbf{y}_{new}$ obtained from $\mathbf{y}_{true}$, $\mathbf{y}_{pred}$. We also allow for filtering the dataset if convenient in defining the metric (as shown in the next example).\\

Suppose we wish to prove that the worst case true positive rates ($TPR$) narrow as we go up the levels. We first filter our dataset $\mathcal{D}$ to contain only rows where $\mathbf{y}_{true}=1$, we temporarily discard other rows. Then create a new binary vector $\mathbf{y}_{truepos}$ which is an indicator on true positives (1 if instance is a true positive, 0 otherwise). Observe that the success rate of this new vector is precisely the true positive rate of the original predictions and true labels. In equations
\begin{equation}
TPR(\mathbf{x}) = \frac{\textrm{\# $true \ positives$}(\mathbf{x})}{\textrm{\# $true$}(\mathbf{x})} = \frac{N^{(1)}(\mathbf{x})}{N(\mathbf{x})}=SR(\mathbf{x})
\end{equation}
By using Theorem 1 with this new dataset and new binary vector, we show true positive rates also narrow as we go up the levels. Using the same approach, we can also show that false positive rates, true negative rates, false negative rates and precision all narrow as we go up the levels.

\section{Average number of datapoints on hypercube at Level $K$}
\label{appendix:avgnum}
We claim the following formula for the average number of datapoints on a hypercube at level $K$:
\begin{equation}
    N_K = \frac{1}{H_K}\sum_{\mathbf{x}\in \textrm{Level }K}N(\mathbf{x})=N2^{M-K}
\end{equation}
\begin{proof}
For $K=0$, the formula holds
    \begin{equation}
    N_0 = \frac{1}{2^M}\sum_{\mathbf{x}\in \textrm{Level }0}N(\mathbf{x})=\frac{N}{2^M}
\end{equation}
For any level $K>0$,
\begin{equation}
    N_K = \frac{1}{H_K}\sum_{\mathbf{x}\in \textrm{Level }K}N(\mathbf{x})=\frac{1}{H_K}\sum_{\mathbf{x}\in \textrm{Level }K}\frac{1}{K}\sum_i \sum_j N(\mathbf{x}_i^{(j)})
\end{equation}
Each hypercube at level $K-1$ appears precisely $M-K+1$ times in the triple sum, so we can rewrite it as
\begin{equation}
    N_K = \frac{1}{H_K} \frac{M-K+1}{K}\sum_{\mathbf{x}\in \textrm{Level }K-1}N(\mathbf{x})
\end{equation}
and since, $H_K = \binom{M}{K}2^{M-K}$, we have
\begin{equation}
     \frac{1}{H_K} \frac{M-K+1}{K} = \frac{K!(M-K)!}{M!}2^{K-M} \frac{M-K+1}{K} = 2\frac{(K-1)!(M-K+1)!}{M!}2^{K-1-M} = 2\frac{1}{H_{K-1}}
\end{equation}
Therefore
\begin{equation}
    N_K = 2\frac{1}{H_{K-1}}\sum_{\mathbf{x}\in \textrm{Level }K-1}N(\mathbf{x})=2 N_{K-1}
\end{equation}
Which concludes the proof by induction.
\end{proof}
\section{Variance of empirical success rates at Level $K$ under ISP assumption}
\label{appendix:varformula}
We claimed that under ISP (and uncorrelated samples), the expected empirical variance can be lower bounded by $\frac{H_K-1}{H_K}\frac{p_{tot}(1-p_{tot})}{N_K}$. The proof is as follows
\begin{proof}
We make 2 assumptions (A1, A2)
\begin{enumerate}
    \item The classifier generating the label data is ISP.
    \item The samples $SR(\mathbf{x}), \mathbf{x}\in \textrm{Level }K$ are uncorrelated.
\end{enumerate}
\begin{equation}
    \mathbb{E}[Var(K)] = \mathbb{E}\left [ \frac{1}{H_K}\sum_{\mathbf{x}\in\textrm{Level }K}SR(\mathbf{x})^2-\left (\frac{1}{H_K}\sum_{\mathbf{x}\in\textrm{Level }K} SR(\mathbf{x})\right )^2 \right ]
\end{equation}
Expand the sum
\begin{equation}
    = \mathbb{E}\left [ \frac{1}{H_K}\sum_{\mathbf{x}}SR(\mathbf{x})^2-\frac{1}{H_K^2}\sum_{\mathbf{x},\mathbf{y}} SR(\mathbf{x})SR(\mathbf{y}) \right ]
\end{equation}
Remove the cross terms that amount to squares, and aggregate them in the first sum
\begin{equation}
    = \mathbb{E}\left [ \left(\frac{1}{H_K}-\frac{1}{H_K^2}\right )\sum_{\mathbf{x}}SR(\mathbf{x})^2-\frac{1}{H_K^2}\sum_{\mathbf{x}\neq\mathbf{y}} SR(\mathbf{x})SR(\mathbf{y}) \right ]
\end{equation}
Now we use linearity of expectation and both assumptions A1, and A2. A1 is used to replace expectations with their theoretical value $p_{tot}$. A2 is used to transform expectations of products into products of expectations.
\begin{equation}
    = \left(\frac{1}{H_K}-\frac{1}{H_K^2}\right )\sum_{\mathbf{x}}\mathbb{E}[SR(\mathbf{x})^2]-\frac{1}{H_K^2}\sum_{\mathbf{x}\neq\mathbf{y}} p_{tot}^2
\end{equation}
\begin{equation}
    = \frac{H_K-1}{H_K^2} \sum_{\mathbf{x}}\mathbb{E}[SR(\mathbf{x})^2]-\frac{H_K-1}{H_K}p_{tot}^2
\end{equation}
\begin{equation}
    = \frac{H_K-1}{H_K^2} \sum_{\mathbf{x}}\mathbb{E}[SR(\mathbf{x})^2]-\mathbb{E}[SR(\mathbf{x})]^2
\end{equation}
\begin{equation}
    = \frac{H_K-1}{H_K^2} \sum_{\mathbf{x}}Var[SR(\mathbf{x})]
\end{equation}
\begin{equation}
    = \frac{H_K-1}{H_K^2} \sum_{\mathbf{x}}\frac{p_{tot}(1-p_{tot})}{N(\mathbf{x})}
\end{equation}
Using Jensen's inequality applied to concave function $1/X$, 
\begin{equation}
    \frac{1}{H_K}\sum_{\mathbf{x}\in\textrm{Level }K} \frac{1}{N(\mathbf{x})} \geq \frac{1}{\frac{1}{H_K}\sum_{\mathbf{x}\in\textrm{Level }K} N(\mathbf{x})} = \frac{1}{N_K}
\end{equation}
So we obtain
\begin{equation}
    \mathbb{E}[Var(K)] \geq  \frac{H_K-1}{H_K}\frac{p_{tot}(1-p_{tot})}{N_K}
\end{equation}
And the bound is tight iff all sample sizes are the same on the given level.
\end{proof}


\section{Metrics for Experiment 2 measured across all levels}
\label{appendix:metrics_big_table}

\begin{table}[H]
    \centering
\begin{tabular}{|c|c|c|c|c|c|c|c|c|c|c|} 
\hline
  \textbf{ } & \textbf{$\delta=0$} & \textbf{$\delta=0.1$} & \textbf{$\delta=0.2$} & \textbf{$\delta=0.3$} & \textbf{$\delta=0.4$}\\ 
\hline
Level 0 &  0.989688 &   1.735759 &    4.191394 &    7.966448 &   13.347224 \\
\hline
Level 1 &  0.993014 &   2.219624 &    6.220185 &   12.409443 &   21.233369 \\
\hline
Level 2 &  0.995939 &   3.001205 &    9.507072 &   19.598287 &   34.003209 \\
\hline
Level 3 &  1.000630 &   4.251804 &   14.776882 &   31.110746 &   54.461175 \\
\hline
Level 4 &  1.009613 &   6.221990 &   23.086325 &   49.241288 &   86.686218 \\
\hline
Level 5 &  1.024945 &   9.240775 &   35.823894 &   77.001449 &  136.032915 \\
\hline
Level 6 &  1.046565 &  13.636296 &   54.377402 &  117.403449 &  207.857395 \\
\hline
Level 7 &  1.069158 &  19.405291 &   78.747043 &  170.448156 &  302.164360 \\
\hline
Level 8 &  1.074935 &  25.157049 &  103.123825 &  223.490712 &  396.491092 \\
\hline
Level 9 &  1.013403 &  25.073643 &  103.158698 &  223.463451 &  396.570403 \\
\hline
\end{tabular}
\caption{\textbf{Experiment 2 Metrics Table}. We show the values of the metric $VarRatio(K)$ for $K=0,\hdots,9$ and parameter $\delta=0,\hdots,0.4$}
\label{tab:metricstable}
\end{table}

\end{document}